\documentclass[letterpaper, 10 pt, conference]{ieeeconf}  

\IEEEoverridecommandlockouts                              

\usepackage{todonotes}
\usepackage{amsmath}
\usepackage{amssymb}
\usepackage{diagbox}
\usepackage{makecell}
\usepackage{amsmath}
\usepackage{mathtools} 
\usepackage{url}
\usepackage{cite}
\usepackage{hyperref}
\usepackage{bm}
\usepackage{array}
\let\labelindent\relax
\usepackage{enumitem}
\usepackage{color,soul}
\usepackage{leftindex}
\usepackage{graphicx}
\usepackage{tensor}   
\usepackage{stackengine}
\usepackage{algorithm}
\usepackage{algpseudocode}
\usepackage{stfloats}

\newtheorem{assumption}{Assumption}

\newtheorem{theorem}{Theorem}

\makeatletter
\newcommand{\newparallel}{{\mathrel{\mathpalette\new@parallel\relax}}}
\newcommand{\new@parallel}[2]{%
  \begingroup
  \sbox\z@{$#1T$}
  \resizebox{!}{\ht\z@}{\raisebox{\depth}{$\m@th#1/\mkern-5mu/$}}%
  \endgroup
}
\makeatother

\title{\LARGE \bf
Self-Supervised Meta-Learning for All-Layer DNN-Based \\ Adaptive Control with Stability Guarantees
}

\author{Guanqi He, Yogita Choudhary, and Guanya Shi
\thanks{The authors are with the Robotics Institute, Carnegie Mellon University, USA.
        {\tt \{guanqihe, ychoudha, guanyas\}@andrew.cmu.edu}}%
\thanks{$^1$ Experimental videos are in project website: \href{https://sites.google.com/view/ssml-ac-project}{https://sites.google.com/view/ssml-ac-project}}
\thanks{We would like to express our gratitude to Prof. Wennie Tabib at Resilient Intelligent Systems Lab for providing the testing field and equipment. Special thanks to Chaoyi Pan for providing the crazyflie hardware codebase and for engaging in many insightful discussions.}
}

\begin{document}

\maketitle
\thispagestyle{empty}
\pagestyle{empty}


\begin{abstract}
A critical goal of adaptive control is enabling robots to rapidly adapt in dynamic environments. 
Recent studies have developed a meta-learning-based adaptive control scheme, which uses meta-learning to extract \emph{nonlinear} features (represented by Deep Neural Networks (DNNs)) from offline data, and uses adaptive control to update \emph{linear} coefficients online. 
However, such a scheme is fundamentally limited by the linear parameterization of uncertainties and does not fully unleash the capability of DNNs.
This paper introduces a novel learning-based adaptive control framework that pretrains a DNN via self-supervised meta-learning (SSML) from offline trajectories and online adapts the \emph{full} DNN via composite adaptation.
In particular, the offline SSML stage leverages the time consistency in trajectory data to train the DNN to predict future disturbances from history, in a self-supervised manner without environment condition labels. The online stage carefully designs a control law and an adaptation law to update the full DNN with stability guarantees.
Empirically, the proposed framework significantly outperforms (19-39\%) various classic and learning-based adaptive control baselines, in challenging real-world quadrotor tracking problems under large dynamic wind disturbance$^1$.

\begin{figure*}[h]
    \centering
    \includegraphics[width=0.99\linewidth]{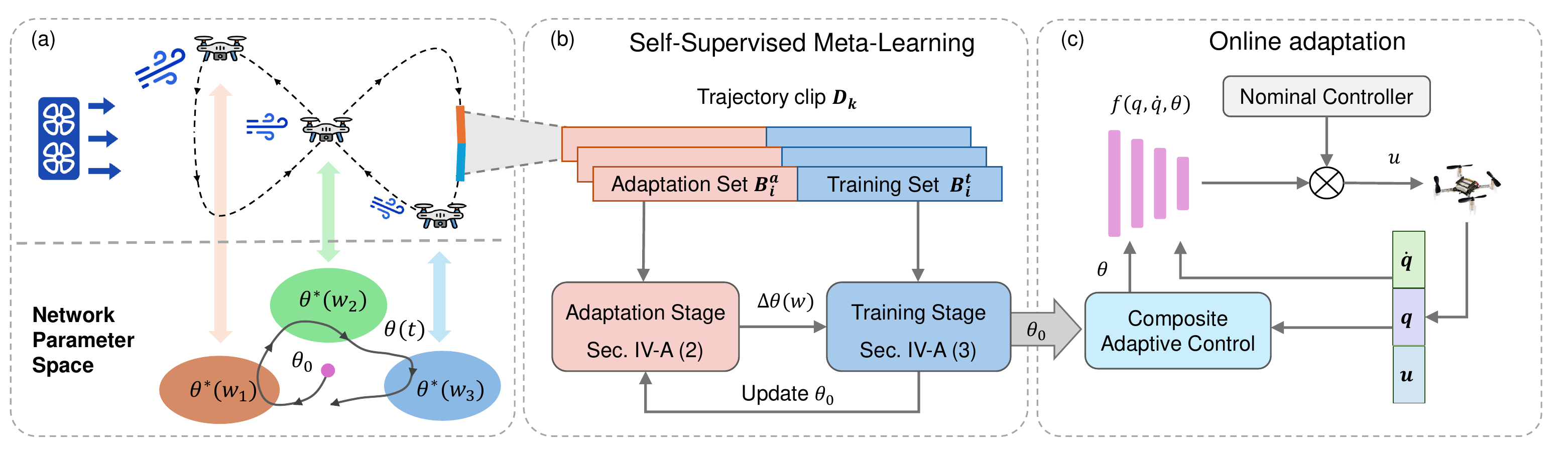}
    \vspace{-0.3cm}
    \caption{(a) The goal of pretraining is to find an initial DNN parameter $\theta_0$ such that, for all environment conditions $w$, $\theta_0$ is close to the optimal DNN parameter $\theta^*(w)$. (b) Pretraining $\theta_0$ via self-supervised meta-learning (SSML) from offline trajectories. (c) Adapt the DNN parameter $\theta$ online using composite adaptive control. }
    \label{fig:long-diagram}
\end{figure*}


\end{abstract}

\section{Introduction}


Real-time adaptive control plays a pivotal role in enabling robots to operate in dynamic and complex environments, including drones in wind \cite{shi2019neural}, manipulators with unknown objects \cite{zeng2019tossingbot}, and quadrupeds on uneven terrains \cite{Lee2020Quadrupedal}.

Adaptive control in general focuses on handling parametric uncertainties with theoretical stability guarantees, where the model structure of the disturbance is assumed to be known, but a set of parameters in the given model is unknown. 
In particular, recent works have leveraged meta-learning to pretrain nonlinear features using offline data. Once these nonlinear features are learned, various adaptive control algorithms, such as direct adaptive control \cite{slotine1992direct}, composite adaptive control \cite{O_Connell_2022} and model reference adaptive control \cite{joshi2019deepmodelreferenceadaptive} can be employed to adapt the linear coefficients.
However, those works are fundamentally limited by the linear parameterization of model uncertainties, and do not fully unleash the capability of DNNs in online adaptation. 

On the other hand, recent works in robot learning and model-based reinforcement learning have demonstrated the effectiveness of meta-learning-based full DNN adaptation \cite{nagabandi2019learningadaptdynamicrealworld, song2020rapidlyadaptableleggedrobots}. However, these works lack the necessary theoretical stability analysis, posing challenges in stability and robustness.

In this paper, we bridge this gap by introducing a novel Self-Supervised Meta-learning for All-Layer DNN-based Adaptive Control (SSML-AC) framework. This framework pretrains a DNN via self-supervised meta-learning (SSML) from offline trajectories and adapts the full DNN online using composite adaptation, as shown in Fig. \ref{fig:long-diagram}.
In particular, the offline SSML stage leverages time consistency in trajectory data to train the DNN to predict future disturbances in a self-supervised manner, without requiring environment condition labels. The online stage carefully designs a control law and a speed-gradient composite adaptation law to update the full DNN, with theoretical stability guarantees. Our contribution is three-fold:
\begin{itemize}
    \item \textbf{SSML-AC Framework}: We propose SSML-AC framework, incorporating the self-supervised meta-learning stage for DNN pretraining and the adaptation law for online full DNN update.
    \item \textbf{Stability Analysis}: Our Lyapunov analysis not only guarantees exponential stability of the closed-loop system, but also guides the loss design and hyper-parameter tuning in the pretraining process.
    \item \textbf{Experimental Validation}: Our method demonstrates state-of-the-art quadrotor adaptive control performance in challenging real-world tracking tasks under highly dynamic wind disturbances, outperforming classic and learning-based adaptive control baselines by 19-39\%.
\end{itemize}


The paper is organized as follows: Section \ref{related-works} reviews relevant background on adaptive control and meta-learning. Section \ref{problem-formulation} formulates the problem of meta-learning for DNN-based online adaptation. In Section \ref{frame-work}, we introduce the SSML-AC training and adaptation method, followed by stability analysis in Section \ref{sec:stability-analysis}. Section \ref{experiment} presents real-world experiments on a Crazyflie quadrotor to demonstrate the effectiveness of SSML-AC framework. Finally, Section \ref{conclusion} concludes the paper.

\section{Related Works} \label{related-works}
    \subsection{Adaptive Nonlinear Control} Adaptive control for handling parametric uncertainty has been extensively studied over the past few decades \cite{relatedwork_nonadapt, relatedwork_nonadapt2, robustadaptivecontrol}, with basis functions playing a crucial role in control performance. 
    Many studies focus on unknown linear coefficients for known basis functions, proposing various adaptive control algorithms such as direct adaptive control\cite{sobel1994directadaptive}, indirect adaptive control\cite{sastry1989indirectadaptivecontrol}, composite adaptive control\cite{SLOTINE1989compositeadaptive} and speed-gradient adaptive control\cite{romeo2008speedgradient}. 
    However, designing efficient basis functions is often challenging and inaccurate in practice. 
    An alternative approach utilizes random features as the basis functions like random Fourier features \cite{benrecht2007randomfeatures} and radial basis functions \cite{slotine1992radialbasisrecursiveidentification}. 
    While by passing the basis function design, they are in general suboptimal and redundant, slowing down the convergence speed of the adaptation process. 
    Some multirotor flight control methods, such as Incremental Nonlinear Dynamic Inversion (INDI) \cite{karaman2018indi} and L1 adaptive control \cite{Mallikarjunan2012L1AC}, directly estimate the disturbance, but these controllers are limited by lagging in disturbance estimation and noise amplification. 
    In this work, we leverage advances in deep learning to pretrain a DNN disturbance representation and empoly composite speed-gradient adaptive control for DNN-based online adaptation.
    \subsection{Adaptive Control with Neural Networks}  
    Neural network-based adaptive control has also been widely studied. 
    Early works focused on shallow or single-layer neural networks \cite{johnson2003limited,NAKANISHI200571, Nakanishi2002learning}, while recent approaches shift towards using deep neural networks to model complex dynamics \cite{abbell2015helicopter, Abbeel2010AutonomousHA} or constrain network Lipschitz gain to ensure stability in DNN-based adaptive control\cite{zhou2021bridging}. 
    Since neural network controllers require appropriate initialization for convergence \cite{Khalil1995}, some studies pretrain DNNs before deployment to enhance control performance \cite{Shi_2019, sánchezsánchez2016realtimeoptimalcontroldeep, Michael2014DirectInverse}. 
    However, typical pertaining methods simply aggregate data and learn average disturbance, so the resulting neural networks lack rapid adaptability, which is critical in dynamic environments.
    \subsection{Meta-learning} 
    Meta-learning is a technique for training efficient models across different tasks or environments \cite{finn2017MAML, hospedales2020metalearningneuralnetworkssurvey, xie2023hierarchical}. Typically, the algorithm solving a specific task is the base-learner, while the algorithm optimizing the meta-objective is the meta-learner. In adaptive control, the adaptive tracking controller serves as the base-learner, with the average tracking error across trajectories as the meta-objective, and the controller's initial parameters as the meta-parameters. Many meta-learning-based adaptive controllers are regression-oriented \cite{O_Connell_2022, xie2023hierarchical, nagabandi2019learningadaptdynamicrealworld,shi2021meta,cui2023leveraging}, whereas others emphasize control-oriented meta-learning \cite{richards2022controlorientedmetalearning, tang2024metalearningadmirrordescent}. 
    Most meta-learning-based adaptive control algorithms assume disturbances are linearly parameterized by nonlinear features \cite{O_Connell_2022, richards2021adaptivecontrolorientedmetalearningnonlinearsystems, tang2024metalearningadmirrordescent}, which limits the capability of DNNs in online adaptation. Other methods fine-tune full network parameters online \cite{song2020rapidlyadaptableleggedrobots, nagabandi2019learningadaptdynamicrealworld}, but lack theoretical stability analysis. Our work bridges these approaches: SSML-AC adapts full network parameters while ensuring exponential convergence during online adaptation.


\section{Problem Statement} \label{problem-formulation}



We consider a general robotic system \begin{equation} \label{full-dynamics}
    M(q)\ddot q + C(q, \dot q)\dot q + g(q) = u + d(q, \dot q, w)
\end{equation} where $q$, $\dot q$, $\ddot q\in\mathbb{R}^n$ denote the pose, velocity and acceleration vectors, respectively. $M(q)$ is the symmetric, positive definite mass and inertia matrix, $C(q, \dot q)$ is the Coriolis matrix, and $g(q)$ is the gravitational force vector. $u\in\mathbb{R}^n$ is the control input and $d(q, \dot q, w)$ includes all unmodeled disturbances, with $w$ representing unknown environment condition. Given a target trajectory $q_d\in\mathbb{R}^n$, our goal is to design a control law $u$ that ensures robot pose $q$ converges to the reference trajectory $q_d$, despite the presence of unknown disturbance $d$.

Since the environment condition $w$ is unknown and time-varying, we assume the disturbances $d$ can be approximated by neural network $f$ with robot states $q$, $\dot q$ as input,
\begin{equation}\label{dist-approx}
    d(q, \dot q, w) \approx f(q, \dot q, \theta^*(w))
\end{equation} where $\theta^*(w)\in\mathbb{R}^p$ represents the network parameters depending on $w$, and the network is over-parameterized ($p\gg n$). Thus, for each disturbance $d$, there exists at least one $\theta^*(w)$ such that Eq. \eqref{dist-approx} holds. Our goal is to design an adaptive law \begin{equation} \label{eq:general-adaptive}
    \dot{\theta} = h(q, \dot q, u, \theta), \quad {\theta}(0) = \theta_0
\end{equation} where $\theta_0$ is the initial network parameter, and the adaptive law $h$ estimates $\theta^*(w)$ in real-time, compensating for $d$ to ensure the robot pose converges to the reference trajectory.

Given the over-parameterization of $f$, for all possible environment condition $w$, we can rewrite Eq. \eqref{dist-approx} as:
\begin{equation}\label{residual-form}
    d(q, \dot q, w) \approx f(q, \dot q, \theta_0 + \Delta \theta)
\end{equation} where $\Delta\theta(w)=\theta - \theta^*(w)$ denotes the network parameter residual that requires online adaptation. This shows that the quality of the initial DNN parameter $\theta_0$ affects the residual $\Delta\theta$, and thus the tracking performance (see Fig. \ref{fig:long-diagram}(a)).

The goal of pretraining is to find a good initial DNN parameter $\theta_0$ that is close to $\theta^*(w)$ across all environment conditions $w$, thereby minimizing the initial parameter residual $\Delta\theta(w)$.






\section{SSML-AC Framework}\label{frame-work}



Although recent meta-learning approaches for adaptive control \cite{shi2019neural, xie2023hierarchicalmetalearningbasedadaptivecontroller, richards2022COML, tang2024mirrodescent} address trajectory tracking under dynamic disturbances, these methods rely on linear parameterization with nonlinear features of disturbance structure, limiting the potential of DNNs in online adaptation. 
We propose the Self-Supervised Meta-Learning for All-Layer DNN-based Adaptive Control (SSML-AC) framework, which pretrains the DNN via self-supervised meta-learning \ref{meta-learning} and adapts the full DNN online using speed-gradient composite adaptation \ref{sec:online=adaptation}.

\subsection{Self-Supervised Meta-Learning}\label{meta-learning}

\begin{figure}
    \centering
    \includegraphics[width = 0.9\linewidth]{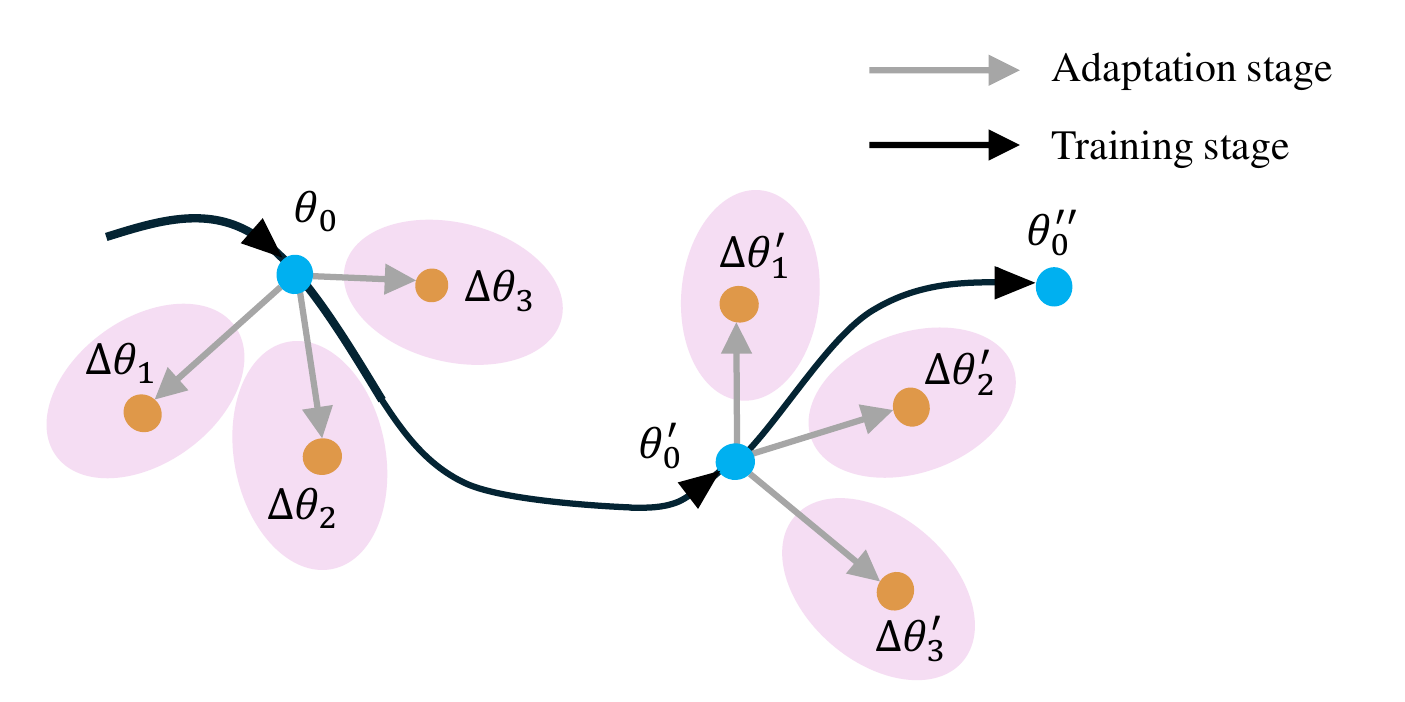}
    \vspace{-0.3cm}
    \caption{Self-supervised meta-learning (SSML) alternates between the adaptation stage and the training stage, pretraining the DNN initial parameter $\theta_0$ for rapid online adaptation.}
    \label{fig: meta-parameter}
\end{figure}

In this section, we introduce Self-Supervised Meta-Learning (SSML). Following previous works \cite{xie2023hierarchical}, we assume that $w$ changes continuously along the trajectory and leverage the temporal continuity of environment conditions for dataset generation and self-supervised learning. Pre-collected trajectories are divided into short clips, each split into adaptation and training sets. Inspired by \cite{finn2017MAML}, the SSML algorithm follows a Bi-level optimization structure with an adaptation stage as the inner loop and a training stage as the outer loop (Fig. \ref{fig:long-diagram} (b))

During the adaptation stage, gradient descent-based adaptation is applied to the adaptation set, computing the network parameter residual $\Delta\theta$. In the training stage, the prediction loss of the network $f(q, \dot q, \theta_0 + \Delta\theta)$ is evaluated on the training set to update $\theta_0$. 
The SSML training process is visualized in Fig. $\ref{fig: meta-parameter}$ and outlined in Algorithm $\ref{alg:SLM}$. 


\subsubsection{Data Collection}

The SSML task dataset is defined as \begin{equation}
\begin{aligned}
    \mathcal{D}_k&=\{\mathcal{B}^a_k, \mathcal{B}^t_k\}\\
    \mathcal{B}^a_k &= \{x^{(j)}, y^{(j)}\}_{j=k}^{k+H_a}\\
    \mathcal{B}^t_k &= \{x^{(j)}, y^{(j)}\}_{j=k+H_a+1}^{k+H_a+H_t}\\
\end{aligned}
\end{equation} where $\mathcal{B}^a_k$ and $\mathcal{B}^t_k$ denote the adaptation and training sets, respectively, with $(x^{(j)}, y^{(j)})$ being the robot state and disturbance measurement at step $j$. The disturbance measurement is obtained via numerical differentiation with robot's nominal dynamics model \cite{miyato2018spectralnormalizationgenerativeadversarial}. $H_a$ and $H_t$ represent the adaptation and prediction horizons for each trajectory clip. The entire meta-learning dataset is constructed as \begin{equation}
    \mathcal{D} = \{\mathcal{D}_k\}_{k=1}^{N}
\end{equation}



\subsubsection{Adaptation stage}
The regression loss $\mathcal{L}(\theta, \mathcal{B})$ on dataset $\mathcal{B}$ is defined as \begin{equation}
    \mathcal{L}(\theta, \mathcal{B}) = \sum_{x^{(j)}, y^{(j)}\in \mathcal{B}} \|y^{(j)} - f(x^{(j)}, \theta)\|^2
\end{equation}

The \emph{adaptation stage} performs gradient descent on each adaptation set $\mathcal{B}^a_k$ to update $\Delta\theta_k$: 
\begin{align}\label{eq:inner-update}
    \Delta\theta_k' &= \Delta\theta_k - \alpha \nabla_{\theta}\mathcal{L}(\theta_0, \mathcal{B}^a_k)
\end{align}
 with adaptation learning rate $\alpha$ as a hyper-parameter. Although single-step gradient updates are consider, multi-step updates are also applicable.


\subsubsection{Training Stage}
The \emph{training stage} update $\theta_0$ by computing the prediction loss over a batch of tasks. The meta-objective is formulated as the prediction loss incorporating regularizations:
\begin{align}
    \mathcal{L}_{\text{meta}}(\theta_0) = & \sum_{\mathcal{D}_k\sim p(\mathcal{D})}\left[\mathcal{L}(\theta_0 + \Delta\theta_k, \mathcal{B}^t_k) \right.\\ &\left.+ \lambda_{\text{dir}} \mathcal{L}(\theta_0, \mathcal{B}^t_k)\right] + \lambda_{\text{norm}}\|\theta_0\|^2
\end{align} 
where $\lambda_\text{dir}$ and $\lambda_\text{norm}$ are regularization coefficients for direct prediction cost and $\theta_0$ norm cost, respectively. The direct prediction cost ensures the model's disturbance prediction capability without adaptation, consistent with the regularization term of the adaptive controller \eqref{adaptive-law}. The norm cost helps reduce the network's Jacobian norm \cite{allenzhu2020learninggeneralizationoverparameterizedneural}, thereby minimizing the disturbance terms in Eq. \eqref{eq:general-disturbance}. Thanks to the overparameterization nature of DNN, these regularization terms do not conflict with the primary objective of disturbance prediction.

Meta-optimization across tasks is performed via stochastic gradient descent (SGD), updating $\theta_0$ as follows:
\begin{equation}
    \theta_0 \leftarrow \theta_0 - \beta \nabla_\theta \!\sum_{\mathcal{D}_k\sim p(\mathcal{D})}\!\left[\mathcal{L}(\theta_k(\theta_0), \mathcal{B}^t_k) + \lambda \mathcal{L}(\theta_0, \mathcal{B}^t_k)\right]
\end{equation} with $p(\mathcal{D})$ as the task distribution. 
To enhance the stability of the neural network during online adaptation, spectral normalization is applied to $\theta_0$, constraining the network's Lipschitz constant\cite{Shi_2019}. 

\begin{algorithm}
    \caption{Self-Supervised Meta-Learning}
    \label{alg:SLM}
    \begin{algorithmic}
        \State Input: Trajectory dataset $\mathcal{D} = \{\mathcal{D}_{k}\}$
        \State Initialize: Neural network weight ${\theta_0}$
        \State Result: Trained neural network weight ${\theta_0}$
        \While{not done}
            \State Sample batch of tasks $\mathcal{D}_k$ from $\mathcal{D}$
            \For {\textbf{all} $\mathcal{D}_k$}
                \State $\mathcal{B}^a_{k}$, $\mathcal{B}^t_{k}\leftarrow \mathbf{D}_k$
                \State $\theta_k\leftarrow\theta_0$\Comment{Gradient descent on $\mathcal{B}^a_k$}
                    \State $\Delta\theta_k \leftarrow - \alpha \nabla_\theta \mathcal{L}(\theta_0, \mathcal{B}^a_k)$
            \EndFor
            \State  $\mathcal{L}_{\text{meta}}(\theta_0) = \sum_{\mathcal{D}_k\sim p(\mathcal{D})}\left[\mathcal{L}(\theta_0 + \Delta\theta_k, \mathcal{B}^t_k) +\right.$ \\ \hspace{6em} $\left. \lambda_{\text{dir}} \mathcal{L}(\theta_0, \mathcal{B}^t_k)\right]+ \lambda_{\text{norm}}\|\theta_0\|^2$
            \State $\theta_0 \leftarrow \theta_0 - \beta\nabla_\theta \mathcal{L}_{\text{meta}}(\theta_0)$ \Comment{Update $\theta_0$}
        \EndWhile
    \end{algorithmic}
\end{algorithm}

\subsection{Adaptive Control or Composite adaptation}\label{sec:online=adaptation}

In the online adaptation stage, the goal is to minimize position tracking error. 
Given a twice-differentiable reference trajectory $q_d$, we define the composite velocity error \begin{equation}
    s = \dot{\Tilde{q}}+\Lambda \Tilde{q} = \dot q - \dot q_r \label{eq:composite-vel-err}
\end{equation} where $\Tilde{q} = q - q_d$ is the position tracking error, $\Lambda\in\mathbb{R}^{n\times n}$ is a positive definite matrix, and $q_r$ is the reference position. 
The nominal PD controller with disturbance compensation is defined as 
\begin{align}
    u = M(q)\ddot q_r+ C(q)\dot q_r + g(q) - K s - f(q, \dot q, \theta) \label{eq:control-law}
\end{align}
We propose the following adaptive controller for the system \eqref{full-dynamics} and nominal feedback controller \eqref{eq:control-law} using the velocity gradient descent method from \cite{fradkov1999nonlinear}. 

The adaptive law is
\begin{align}
    \bar{V} &= s^\top M(q) s  + (f(q, \dot q, \theta) - y)^\top\Gamma(f(q, \dot q, \theta) - y)\\
    \dot{\theta}&= -\gamma\nabla_\theta\dot{\bar{V}}(\theta) - \lambda (\theta - \theta_0)\notag\\ 
    & =-\gamma{J}^\top\left [ s + \Gamma(f - y) \right ]\! -\! \lambda (\theta - \theta_0), \quad \theta(0) = \theta_0\label{adaptive-law}
\end{align}
where $\bar V$ is a Lyapunov-like function, $y=d + \epsilon$ is the measured disturbance with measurement noise $\epsilon$, and ${J}=\frac{\partial f}{\partial \theta}(q, \dot q, \theta)$ is the Jacobian of $f$. $K$, $\Gamma$ are both positive definite gain matrices and $\gamma$, $\lambda$ are positive constant.

To constrain the Jacobian of $f$, we apply layer spectral normalization on $\theta$ during online adaptation:
\begin{equation}
    \theta^+_i = \text{Proj}(\theta^-_i) = \left\{\begin{array}{cc}
      \theta^-_i   &  \text{if} \|\theta^-_i\|<\nu\\
      {\theta_i^-}/{\|\theta_i^-\|}   & \text{otherwise}
    \end{array}\right.
\end{equation} 
Thus, the network parameter $\theta$ belongs to a bounded, compact set $\Theta$ during the online adaptation, further bound the Jacobian $J$ if the robot states are bounded.

The control law and adaptive law are designed to ensure the closed-loop system remains robust to imperfect learning and time-varying environment conditions. In the next section, we discuss the theoretical guarantees of the proposed control and adaptive scheme.


\section{Stability Analysis}\label{sec:stability-analysis}

In this section, we present the theoretical stability analysis of the proposed control scheme, showing that the system remains stable under disturbance $d(q, \dot q, w)$, and that the tracking error $\Tilde{q}$ converges to an error ball exponentially fast. 

We make the following assumptions for system \eqref{full-dynamics} and the neural network $f$:
\begin{assumption}\label{assu:taylor}
    For all environment conditions $w$ and DNN parameter $\theta\in\Theta$, there exists an optimal DNN parameter $\theta^*(w)$, such that the disturbance prediction error $f(q, \dot q, \theta^*(w)) - f(q, \dot q, \theta)$ can be approximated by its Taylor expansion and the approximation error $r$ is bounded. \begin{equation}\label{NN-Taylor}
    f(x, \theta) - f(x, \theta^*(w)) = {J} \Tilde{\theta} + r, \quad \|r\|\leq\bar r
\end{equation}  where $\Tilde{\theta} = \theta - \theta^*$ denotes the network parameter error.
\end{assumption}

\begin{assumption}\label{assu:bounded}
    For all $\theta\in\Theta$, the output of the neural network $f(q, \dot q, \theta)$ is bounded.
\end{assumption}

\begin{assumption}\label{assu:dist}
    The disturbance $d$ and neural representation error $\mu(t)=d - f(q, \dot q, \theta^*(w))$ are uniformly bounded.
\end{assumption}

Thanks to the meta-learning-based training strategy in Sec. \ref{meta-learning}, the initial network parameter $\theta_0$ is close to the optimal network weight $\theta^*(w)$ for all environment conditions $w$, ensuring that Assumption \ref{assu:taylor} holds for all environment conditions $w$ in the training set $\mathcal{D}$. Additionally, pretraining minimizes the initial parameter error $\Tilde{\theta}(0)$, accelerating the convergence speed of online adaptation.

Assumption \ref{assu:bounded} can be simply satisfied in practice by clipping the network output, and assumption \ref{assu:dist} is common in existing works \cite{O_Connell_2022, xie2023hierarchical}.


Substituting the nominal control law \eqref{eq:control-law} into the dynamics \eqref{full-dynamics}, the resulting closed-loop dynamics becomes:
\begin{equation}\label{eq:simplified-dynamics}
    M\dot s + (C+K) s = - f(q, \dot q, \theta) + d(q, \dot q, w)
\end{equation} 

The next step is to prove the boundedness of the network Jacobian $J$, which will be applied in Theorem \ref{theo:ES}.

\begin{theorem}
    Suppose the network $f$ and disturbance $d$ satisfy Assumptions \ref{assu:bounded} and \ref{assu:dist}. The states $s$, $q$, $\dot q$ of the closed-loop system \eqref{eq:simplified-dynamics} are bounded. Furthermore, there exists a constant $\bar J$ as the uniform upper bound of the Jacobian matrix $\|J\|$ for the network $f(q, \dot q, \theta)$.
\end{theorem}



\begin{proof} 
Since the input and disturbance to system \eqref{eq:simplified-dynamics} are bounded, we can prove the boundedness of the states $s$, $q$, $\dot q$ via Input-to-State Stability (ISS) \cite{slotine1991applied, robustadaptivecontrol}.

According to \cite{oymak2019generalizationguaranteesneuralnetworks, allenzhu2020learninggeneralizationoverparameterizedneural}, there exists a function $\zeta$ that the network Jacobian $J$ is bounded by input $q$, $\dot q$ and network parameter $\theta$: 
\vspace{-0.1cm}
\begin{equation}
    \|J\|\leq \bar J = {\sup}_t\zeta(\|q\|, \|\dot q\|, \nu)
\end{equation} where $\nu$ is the layer spectral norm bound of network parameter $\theta$.
\end{proof}


Next, we prove that the tracking error $\Tilde{q}$ converges to an error ball exponentially fast.

\begin{theorem}\label{theo:ES} 
Consider the system with dynamics \eqref{full-dynamics}, controller \eqref{eq:control-law}, and adaptation law \eqref{adaptive-law}. Suppose Assumptions \ref{assu:taylor}, \ref{assu:bounded}, and \ref{assu:dist} are satisfied. The parameter estimation error $\Tilde{\theta}$ and trajectory tracking error $\Tilde{q}$ converge to an error ball exponentially fast.
\end{theorem}

\begin{proof}
    Consider the Lyapunov function given by
\begin{equation}
    V = \begin{bmatrix}
        s\\ \Tilde{\theta}
    \end{bmatrix}^\top \begin{bmatrix}
        M & 0 \\ 0 & \gamma^{-1} I
    \end{bmatrix}\begin{bmatrix}
        s \\ \Tilde{\theta}
    \end{bmatrix}
\end{equation} Using the fact that $\dot M - 2C$ is skew symmetric, and substituting $\Tilde{\theta} = \theta - \theta^*$, the adaptation law $\eqref{adaptive-law}$, and the network Taylor expansion $\eqref{NN-Taylor}$, we have \begin{align}\label{eq:dotV}
    \dot V = &-2\begin{bmatrix}
        s \\ \Tilde{\theta}
    \end{bmatrix}^\top \begin{bmatrix}
        K & 0 \\ 0 & {J}^\top \Gamma {J} + \lambda I
    \end{bmatrix} \begin{bmatrix}
        s \\ \Tilde{\theta}
    \end{bmatrix} \notag\\ & - 2\begin{bmatrix}
        s \\ \Tilde{\theta}
    \end{bmatrix}^\top \begin{bmatrix}
        r-\mu - \epsilon \\ {J}^\top\Gamma r + \lambda(\theta^* - \theta_0)
    \end{bmatrix}
\end{align} 
Since $K$, $\lambda I$, $M$ and $\gamma^{-1}I$ are all positive definite matrices, and $J^\top\Gamma J$ is positive semidefinite, there exists a constant $\rho>0$ such that \begin{equation}
    -2\begin{bmatrix}
        K & 0 \\ 0 & {J}^\top \Gamma {J} + \lambda I
    \end{bmatrix} \preceq \rho \begin{bmatrix}
        M & 0 \\ 0 & \gamma^{-1} I
    \end{bmatrix}
\end{equation} 


We define the upper bound for the disturbance term $D$ as \begin{align} \label{eq:general-disturbance}
    D &= \sup_{t}\left\|-2\begin{bmatrix}
        r - \mu \\ {J}^\top\Gamma (r + \epsilon + \mu) + \lambda (\theta^* - \theta_0)
    \end{bmatrix}\right\|
\end{align} where $\epsilon$ is the observation noise of disturbance $d$, and function $\mathcal{M} = \text{diag}(M,\gamma^{-1} I)$.

By Cauchy-Schwartz inequality we obtain the following inequality: \begin{equation}
    \dot V \leq -2 \rho V + 2\sqrt{\frac{V}{\lambda_{\min}(\mathcal{M})}}D
\end{equation} 
Following the procedure in \cite{O_Connell_2022}, 
the tracking error and parameter estimation error exponentially converge to the error ball \begin{equation}
    \lim_{t\rightarrow\infty}\left\|\begin{bmatrix}
        s & \Tilde{\theta}
    \end{bmatrix}^\top \right\|\leq \frac{D}{\rho\lambda_{\min}(\mathcal{M})}
\end{equation} 
Thus, the state tracking error $\Tilde{q}$ also converges to an error ball exponentially fast when $s$ is exponentially stable. 
\end{proof}

\section{Experiments}\label{experiment}
We demonstrate the effectiveness of SSML-AC in challenging real-world quadrotor tracking tasks under large dynamic wind conditions. 

\begin{figure}
    \centering
    \includegraphics[width = 1.0\linewidth]{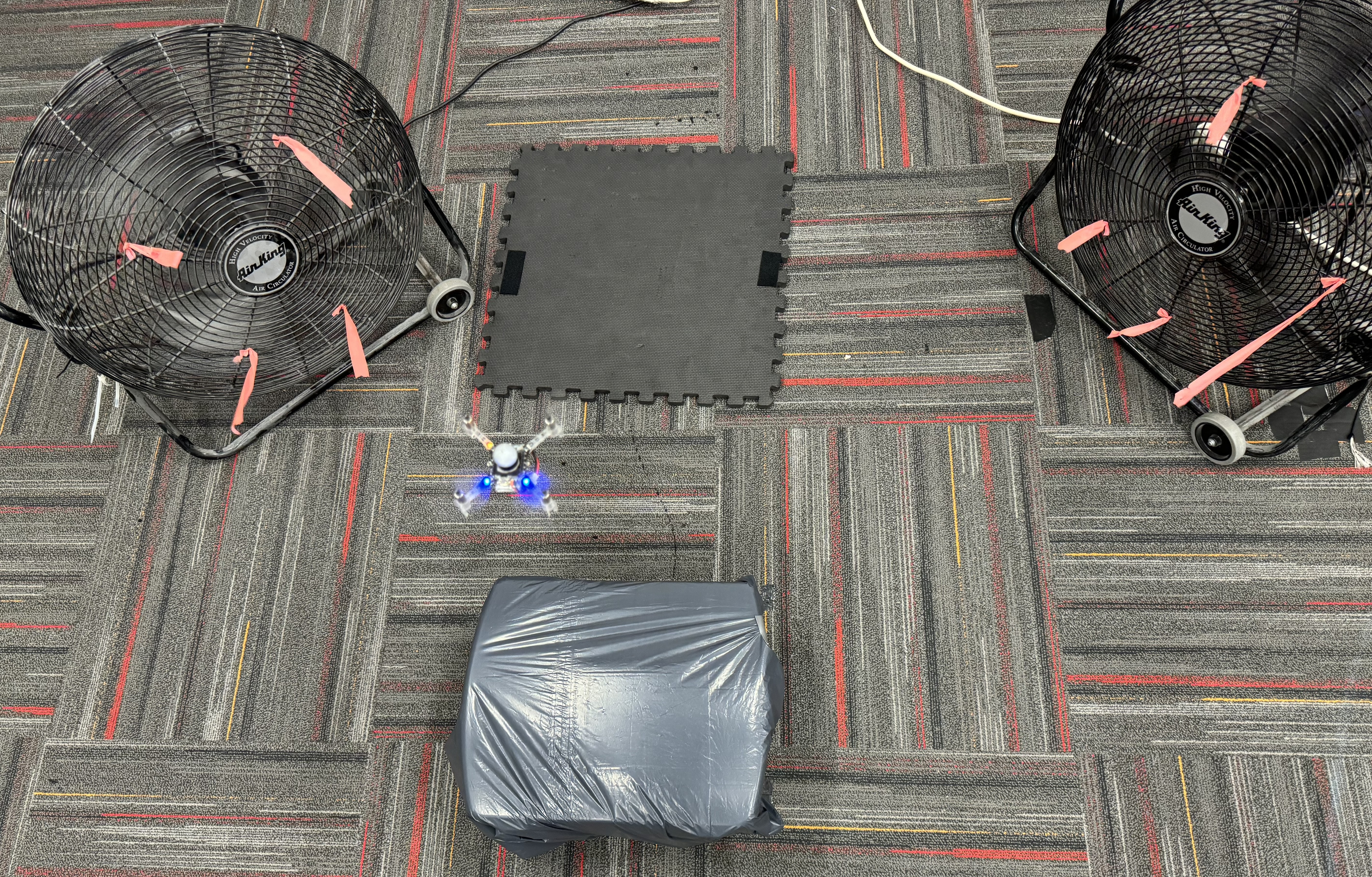}
    \vspace{-0.4cm}
    \caption{Experiments setting of quadrotor trajectory tracking under large dynamic wind conditions.}
    \label{fig:quadrotorfig}
\end{figure}

\subsection{Experiment setup}
 We use a Bitcraze Crazyflie 2.1 enhanced with a thrust upgrade bundle as our quadrotor platform. The quadrotor weighs 42$g$ and has a rotor-to-rotor distance of 9$cm$. The Crazyswarm API\cite{preiss2017crazyswarm} is used for the control and communication with the quadrotor. Full-state estimates of the quadrotor are provided by the Optitrack system, and communication between ground station and quadrotor uses a 2.4 GHz radio, with both state estimates and control commands sending at 50 Hz. The setting of the experiment is to track reference trajectories under the wind disturbance. 
 
 Wind disturbances are generated by two fans, as shown in Fig. \ref{fig:quadrotorfig}. The wind disturbance has an average intensity of $1.0 m/s^2$ (in terms of drone acceleration) and changes rapidly due to the motion of fan blades. 

\subsection{Quadrotor Dynamics}

The dynamics of the quadrotor can be written as \begin{align}
    \dot p &= v, \quad m\dot v  = mg + R f_u + d,\label{pos-dyn}\\
    \dot R&=R S(\omega), \quad J\dot\omega = J \omega \times \omega + \tau_u
\end{align} where $p\in \mathbb{R}^3$, $v\in\mathbb{R}^3$, $R\in SO(3)$ and $\omega\in\mathbb{R}^3$ are position, velocity, attitude rotation matrix and body angular velocity, respectively. $m$ is the mass, $J$ is the inertia matrix of the quadrotor, $g$ is the gravity vector, 
 $S(\cdot)$ is the skew-symmetric mapping, $f_u\in\mathbb{R}$ and $\tau_u\in\mathbb{R}^3$ are the generated thrust and body torques from four rotors, and $d \in \mathbb{R}^3$ represents unmodeled aerodynamic disturbances. 

We reformulate the position dynamics \eqref{pos-dyn} into the equivalent form of $\eqref{full-dynamics}$ by assuming $M(q)=mI$, $C(q, \dot q)=0$ and $u = Rf_u$. Our method operates within the position control loop to compute the desired force $u_d$, which is then decomposed into the desired attitude $R_d$ and thrust $T_d$ using kinematics. These values are sent to the attitude and angular rate controller, where the onboard firmware manages attitude tracking and thrust generation \cite{minimum-snap}.

\subsection{Data Collection and Training Details}
We collect data by having the drone track three randomly generated $60$ seconds trajectories using a nominal PD controller under wind disturbance. The wind disturbances are generated by two fans, and data including robots state and input are collected $x = [q, \dot q, u]$. 
The data is filtered with a 4th-order Butterworth filter with cutoff frequency of $20$ Hz, and acceleration $\ddot q$ is computed using a Five-point stencil approach. This provides a noisy measurement of the unmodelled dynamics $d$, denoted as $y = f(q, \dot q, \theta) + \epsilon$. 
We construct the meta-learning dataset $\mathcal{D}$ from the trajectories $\{x_k^{(i)}, y_k^{(i)}\}_{i=1}^{N}$, and the trajectory label is not required for building $\mathcal{D}$.

We implement the SSML framework in Python using Pytorch \cite{NEURIPS2019_9015}. 
Since we notice that the aerodynamic effects only depend on the UAV velocity $v$, angular velocity $\omega$ and attitude quaternion $q$, the input state of network $f$ is constructed as a $11$-dimensional vector $x = [v, \omega, q, u]$, excluding the drone position to avoid overfitting. The DNN consists of 3 fully connected hidden layers with width $50$ using ReLU activation. Spectral normalization is applied to $\theta_0$ in both adaptation and training stages to constrain the neural network Lipschitz and Jacobian. We set the adaptation and prediction horizons to $25$ steps ($0.5$ second), and train for $50$ epochs with base-task learning rate $\alpha = 0.002$ and meta-task learning rate $\beta = 0.001$. Regularization coefficients are selected as $\lambda_{dir}=0.5$ and $\lambda_{norm}=0.05$.

\begin{table}[]
\caption{Trajectory tracking performance}
\centering\footnotesize
\begin{tabular}{|p{1.4cm}|p{1cm}|p{1cm}|p{1cm}|p{1cm}|p{1cm}|}
\hline
& SSML-AC & SSML-AC-LL & Vanilla-NN & INDI & PID\\
    \hline
    RMSE [cm] & $\textbf{5.3}\pm\textbf{0.5}$ & $6.6\pm 0.5$ & $7.9 \pm 0.9$ & $7.0\pm 0.6$ & $8.7\pm 0.9$\\
    \hline
\end{tabular}
\label{table:comparison}
\end{table}

\subsection{Control Performance in Flight Tests}
We compare SSML-AC approach with four baselines: 
\begin{enumerate}
    \item Vanilla-NN: Pretraining DNN without meta-learning.
    \item SSML-AC Last Layer (SSML-AC-LL): Adapt only the last layer linear coefficient of DNN, similar with Neural-Fly \cite{O_Connell_2022}.
    \item Incremental Nonlinear Dynamic Inversion (INDI)\cite{karaman2018indi}.
    \item PID controller with nominal quadrotor model.
\end{enumerate}
The tracking reference trajectory is a $1.2$m long and $1.0$m wide figure-8 trajectory, repeated $3$ times.

The disturbance prediction and tracking performance is shown in Fig. \ref{fig: dist-prediction}. The experiment results show that the PID controller exhibits poor tracking performance due to the lack of wind disturbance feedforward compensation. The INDI controller estimates disturbances from IMU and motion capture system feedback, but applies a low-pass filter to disturbance estimation, leadning to phase delays and loss of high-frequency disturbance details, as shown in the second row of Fig. \ref{fig: dist-prediction}. For the vanilla-NN model, since the wind disturbance is dependent on the position of the drone, without meta-learning, it only learns an average disturbance prediction, lacking parameter flexibility in online adaptation. 

In contrast, SSML-AC and SSML-AC-LL outperform all other baselines in both disturbance prediction and tracking accuracy, illustrating the importance of meta-learning in pretraining stage. Furthermore, SSML-AC further improves disturbance prediction accuracy and tracking performance, demonstrating the advantage of full-network adaptation. Repeated experiments show that both SSML-AC and SSML-AC-LL have lower variance than other baselines, highlighting the role of accurate disturbance representation in improving the consistency and robustness of online adaptation.

Although SSML-AC significantly outperforms other baselines, its performance is limited by the underactuated nature of quadrotor systems, particularly due to the response speed of the attitude control loop. Additionaly, the complex interaction between wind disturbances and the propellers affects the torque output of the quadrotor, which SSML-AC cannot fully compensate for.

\begin{figure}[h]
    \centering
    \includegraphics[width = 1.0\linewidth]{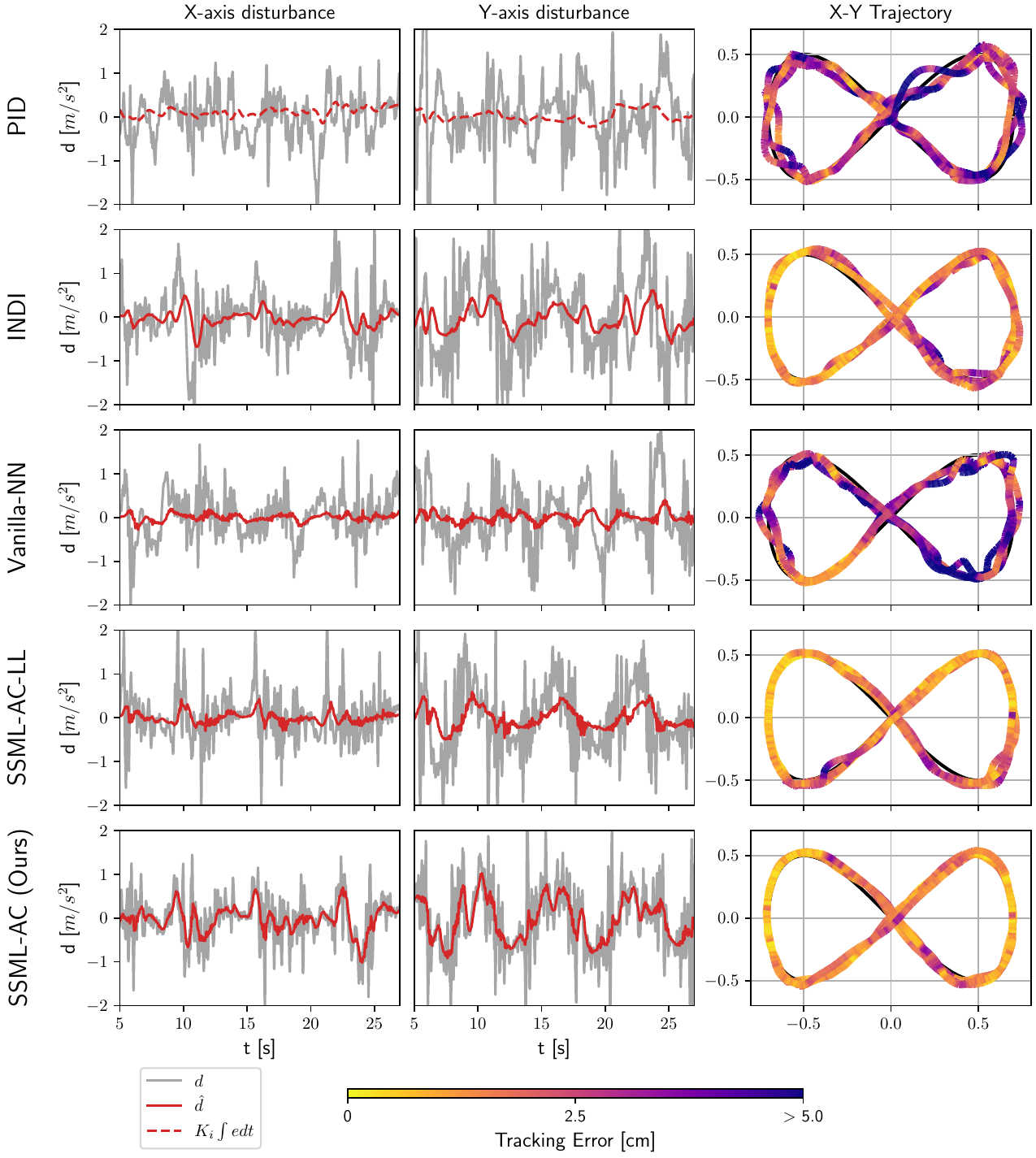}
    \vspace{-0.5cm}
    \caption{Disturbance prediction and tracking performance of each controller. 
    The PID controller struggles with handling unknown wind dynamics, while the INDI controller is affected by acceleration measurement noise and delays in disturbance prediction. The Vanilla-NN model only learns the average disturbance, showing limited flexibility in parameter adaptation. Both SSML-AC and SSML-AC-LL perform well in trajectory tracking; however, SSML-AC achieves more accurate disturbance prediction and lower tracking error (Table \ref{table:comparison}), highlighting the benefits of full-network adaptation. 
    }
    \label{fig: dist-prediction}
\end{figure}



\section{Discussion and Future Work}\label{conclusion}
This paper presents SSML-AC, a Self-Supervised Meta-Learning approach for all-layer DNN-based Adaptive Control. Unlike previous works, our method eliminates the assumption of linear parameterization of uncertainties, allowing DNNs to fully unleash their potential in online adaptation. 
We also provide a rigorous theoretical analysis to ensure the stability of the proposed adaptive controller.
Our framework is validated on challenging real-world quadrotor trajectory tracking tasks under dynamic wind conditions, showing distinct performance improvements over both classic and learning-based adaptive control baselines. 

In the experiments, we did not account for the underactuated nature of the quadrotor, and the wind disturbance is simplified as a matched disturbance for the vehicle position.
Future work will focus on generalizing SSML-AC to handle underactuated systems and unmatched model uncertainties.

\bibliographystyle{IEEEtran}
\bibliography{root}

\end{document}